\documentclass[a4paper]{article}
\usepackage[margin=1in]{geometry}
\usepackage{authblk}
\usepackage{times,helvet,courier}
\usepackage{natbib}

\newcommand{\newmarker}{\textcolor{cyan}{[NEW]}\xspace}
\newcommand{\editmarker}{\textcolor{cyan}{[EDIT]}\xspace}
\newcommand{\edithilight}[1]{\textcolor{cyan}{#1}}
\newcommand{\peter}[1]{#1}
\newcommand{\alc}{\color{cyan}}
\let\newmarker\relax
\let\editmarker\relax
\let\edithilight\relax
\let\alc\relax

\title{\textcolor{black}{(1+1) Genetic Programming With Functionally Complete Instruction Sets Can Evolve Boolean Conjunctions and Disjunctions with Arbitrarily Small Error}}
\author[1]{\bf Benjamin Doerr}
\author[2]{\bf Andrei Lissovoi}
\author[3]{\bf Pietro S. Oliveto}
\affil[1]{doerr@lix.polytechnique.fr\\ Laboratoire d'Informatique (LIX), {\'E}cole Polytechnique, Institut Polytechnique de Paris, Palaiseau, France}
\affil[2]{a.lissovoi@sheffield.ac.uk\\ Department of Computer Science, University of Sheffield, UK}
\affil[3]{p.oliveto@sheffield.ac.uk\\ Department of Computer Science, University of Sheffield, UK}

\usepackage{amsthm}
\newtheorem{theorem}{Theorem}
\newtheorem{lemma}{Lemma}
\newcommand{\proofOf}[1]{Proof of #1}

\usepackage{amsmath,amssymb,xspace}
\usepackage[noend]{algpseudocode}
\usepackage{algorithm}
\usepackage{xcolor}
\usepackage{hyperref}
\newcommand{\andn}{\textsc{AND}\ensuremath{_n}\xspace}
\newcommand{\orn}{\textsc{OR}\ensuremath{_n}\xspace}
\newcommand{\RLSGP}{RLS-GP\xspace}
\newcommand{\OneOneGP}{$(1+1)$~$\mathrm{GP}$\xspace}

\newcommand{\hh}{\ensuremath{\hat{h}}\xspace}
\newcommand{\Tmax}{\ell}
\newcommand{\N}{\mathbb{N}}

\usepackage{tikz}
\usetikzlibrary{trees}
\usepackage{pgfplots}
\pgfplotsset{compat = 1.14}
\colorlet{dred}{red!80!black}
\colorlet{dgreen}{green!80!black}

\begin{document}
\date{}
\maketitle

\begin{abstract}
Recently it has been proven that simple GP systems can efficiently evolve a conjunction of $n$ variables if they are equipped with the minimal required components.
In this paper, we make a considerable step forward by analysing the behaviour and performance of a GP system for evolving a Boolean conjunction or disjunction of $n$ variables using a complete function set that allows the expression of any Boolean function of up to $n$ variables.
First we rigorously prove that a GP system
using the complete truth table to evaluate the program quality, and equipped with both the AND and OR operators and positive literals,
evolves the exact target function in $O(\Tmax n \log^2 n)$ iterations in
expectation, where $\Tmax \geq n$ is a limit on the size of any accepted tree. 
Additionally, we show that when a polynomial sample of possible inputs
is used to evaluate the solution quality,
conjunctions or disjunctions with any polynomially small generalisation error can be evolved
with probability $1 - O(\log^2(n)/n)$.
The latter result also holds if GP uses AND, OR and positive and negated literals, thus has the power to express any Boolean function of $n$ distinct variables.
To prove our results we introduce a super-multiplicative drift 
theorem that gives significantly stronger runtime bounds when the expected progress is only slightly super-linear in the distance from the optimum. 
\end{abstract}

\section{Introduction}
Genetic Programming (GP) uses principles of Darwinian evolution to
evolve computer programs with some desired functionality. The most popular and
well-known GP approach, pioneered by \citet{Koza92book}, represents
programs using syntax trees. It uses genetic variation operators
to search through the space of programs composed of the available components,
favouring ones which exhibit better behaviour on a wide variety of possible
inputs. In this setting, the quality of a program is evaluated by comparing its
outputs for varying inputs to those of the target function.

Despite the many examples of successful applications of GP (see e.g., \peter{\citealp{Koza10,LiuS13,BartoliCLM14,MooreHumies19,MirandaHumies19,LynchHumies19,VuHumies19}}), our understanding of
its behaviour and performance is limited.
 
The few available theoretical analyses of GP have followed the very successful path used in the analysis of traditional evolutionary algorithms (EAs) for function optimisation that initially considered simplified EAs such as the (1+1) evolutionary algorithm \citep{DrosteEtAl2002}, and has progressively allowed the analysis of realistic EAs using populations and crossover (see, e.g., 
\citealp{JansenJW05,Witt06,DoerrHK12,DangEtAlTEVC2018,HuangZCH19,CorusOliveto2020,DangEL21aaai,Sutton2021,CorusLissovoiOlivetoWitt2021,ZhengLD22}).
Similarly simplified GP
systems have been considered that use a tree-based mutation operator called HVL-Prime (an
adaptation of the hierarchical variable length mutation operator proposed by \citealp{MayO96})
to evolve a single program. \textcolor{black}{In this paper we consider \RLSGP \citep{DurrettNO11}, which analogously to the \emph{randomised local search}
algorithm \citep[see, e.g.,][]{NeumannW07}}, performs a single local mutation before evaluating the fitness, differently to 
the \OneOneGP which can perform a larger number of local changes in a 
single iteration, akin to the $(1+1)$ evolutionary algorithm.

Initial works analysed the behaviour and performance of 
these algorithms for the evolution of non-executable tree structures \peter{rather than the evolution of computer programs where the fitness function is an estimate
of how well the the candidate solution matches the behaviour of the target program on a training set of inputs} 
\citep{DurrettNO11,KotzingSNO14,DoerrKLL20,LissovoiOlivetoChapter}.
Only recently the performance of simple GP systems has been analysed for the evolution of executable functions. It has been
proven that Boolean conjunctions of $n$ variables can
be evolved by \RLSGP and (1+1)~GP algorithms in an expected polynomial number
of iterations \citep{MambriniO16,LissovoiO19}. Programs equivalent to the target
conjunction can be evolved when the complete truth table (i.e.\ the set of all $2^n$ possible
inputs) is used to evaluate the program quality. When the solution quality is
evaluated by sampling a polynomial number of inputs uniformly at random from
the complete truth table in each iteration (i.e.\ employing Dynamic Subset
Selection to limit the total computational effort as suggested by
\citealp{GathercoleR94}) the evolved programs are found to generalise well \peter{i.e., they are incorrect only on an arbitrarily small polynomial fraction of the $2^{n}$ possible inputs}.

While these results are promising, the considered GP systems were considerably
different from those used in practice. In particular, they were required to
evolve a simple arity-$n$ Boolean conjunction from only its basic components
(i.e.\ only the \emph{binary} Boolean \textsc{AND} operator and the variables
necessary for the problem). In realistic applications, the required set of
program components is not necessarily known in advance, and thus GP systems
typically have access to a wider range of components than is strictly
necessary. Hence an important question that we address in this paper is whether the GP system is able to cope with a function set containing elements which should be avoided to express the target function concisely or avoided altogether. Ideally, the system should be equipped with at least a complete set of
operators, from which any Boolean function could be constructed.

In this paper, we make a considerable step forward in this direction by analysing the behaviour
and performance of \RLSGP for evolving an unknown Boolean function. 
More precisely, 
the target functions we consider are either \andn, the
conjunction of $n$ variables \peter{or \orn, the disjunction of $n$ variables.}
The GP system has access to both the binary
conjunction (i.e.\ \textsc{AND}) and disjunction operators (i.e.\ \textsc{OR}).
Using \andn or \orn as the target function simplifies our understanding of the quality
of candidate solutions that mix conjunction and disjunction operators.
\peter{Furthermore, we also consider a scenario where the terminal set contains all the $n$ variables in both positive and negated forms.
Thus, such a GP system is {\it complete} for the Boolean domain, as with its function and terminal sets it may express any possible Boolean function of $n$ variables.}

\peter{These} more complex problem setting\peter{s} induce 
 us to introduce more sophisticated
features into the \RLSGP system than those necessary to evolve conjunctions
using the \textsc{AND} operator alone, thus making the GP system more similar
to the ones used in realistic applications.
Since the presence of disjunctions in the current
solution may reduce the effectiveness of the mutation operator for producing
programs with better behaviour, 
we introduce a limit on the size of the syntax tree. This allows us to avoid issues due
to bloat (a common problem for GP systems, where the size of the solution
tends to increase without a corresponding increase in solution quality
\citep{Koza92book,PoliLM2008}). 

With the limit on the tree size in place, our theoretical analysis reveals that
the HVL-Prime mutation operator used in previous work
\citep{DurrettNO11,LissovoiO19}, which either inserts, substitutes or deletes
one node of the tree, may get stuck in local optima. Hence, \RLSGP with the
traditional HVL-Prime operator has infinite expected runtime. To avoid this
issue, we introduce a mutation mechanism which is more similar to the most
commonly used subtree mutation \citep{Koza92book,PoliLM2008}. Specifically, it
allows the deletion operator to remove entire subtrees in one operation, rather
than limiting it to only a single leaf and its immediate parent.

We \peter{first} show that \RLSGP with the above modifications is able to cope efficiently
with the extended function set \peter{and the positive literals in the terminal set if it uses the complete
truth table to evaluate the program quality and rejects any tree with more than
$\Tmax = (1+c)n$ (where $c>0$ is a constant) leaf nodes}. In particular, we prove that it evolves the exact
target functions in $O(\Tmax n \log^2 n)$ iterations in expectation. While using
the complete truth table to evaluate the program quality requires
exponential time in the number of variables, we consider this setting for two main reasons. First, 
this setting represents the best-case model of the GP system's behaviour
(i.e.\ a system unable to find the optimal solution when given access to a reliable
fitness function is unlikely to be able to perform well with a noisy one).
Second,
the deterministic fitness values somewhat simplify the behaviour of the algorithm
and hence our analysis. \peter{Note that if the negated literals are also included in the terminal set, then it has been proven 
that the standard \RLSGP cannot evolve \andn in polynomial time with overwhelming probability even if only the AND operator is used \citep{MambriniO16}.
We conjecture that the same holds for the modified \RLSGP with extended function and terminal sets, and provide experimental evidence that this is the case.}

\peter{On the other hand, we provide a positive general result for the more realistic scenario where training sets of polynomial size are sampled} in each
iteration uniformly at random from the complete truth table. In practice some
information about the function class to be evolved may be used to decide which
inputs to use in the training set. For instance, if the target function was
known to be the conjunction of $n$ variables, then a compact training set of
linear size would suffice to evolve the exact solution
efficiently \citep{LissovoiO19}. However, we assume that the target function is
an unknown arbitrary function composed of conjunctions and disjunctions of $n$
variables. Our aim is to estimate the quality of the solution produced by
\RLSGP in this setting.

We show that with probability $1 - O(\log^2(n)/n)$ \RLSGP is
able to construct and return a conjunction \peter{(or a disjunction)} with
a\peter{n arbitrarily small} polynomial generalisation error in a logarithmic number of iterations \peter{even if the negated literals are present in the terminal set}.

To achieve our results, we introduce a super-multiplicative drift theorem that
makes use of a stronger drift than the linear one required by the traditional
multiplicative drift theorem \citep{DoerrJW12}. This new contribution to the
portfolio of methodologies for the analysis of randomised search heuristics
\peter{\citep{LehreO17bookchapter,LenglerChapter}} allows for the achievement of
drastically smaller bounds on the expected runtime in the presence of a strong
multiplicative drift.

We complement our theoretical results with an empirical investigation that confirms
our theoretical intuition that leaf-only deletion may get stuck on local optima if a limit on the 
tree size is imposed for bloat control reasons. Additionally, the experiments
indicate that while the algorithm would evolve the solution more quickly
without a limit on the tree size, the size limit reduces the amount of expected
undesired binary disjunction operators in the returned program. 

\newmarker
A preliminary version of this work
has previously been published at the Genetic and Evolutionary
Computation Conference \citep{DoerrLO19}. In this version, we additionally show
that the considered GP system \peter{without
further modifications} can \peter{also} efficiently evolve disjunctions, and  \peter{provide the more general result that the algorithm is also efficient when using an ideal function and terminal set} which allows \peter{for} any
Boolean function \peter{of up to $n$ distinct variables} to be represented.
\peter{We have also extended the experimental analyses to complement the theory in this more general setting.}

\section{Preliminaries} \label{sec:Prelim}

\editmarker In this work, we will analyse the performance of the simple \RLSGP
algorithm on the \andn and \orn problems: the former directly, and the latter
by noting that equivalent results can be derived by observing a symmetry
between the search spaces of the two problems. The \andn problem concerns the
evolution of a conjunction of all $n$ input variables while using $F =
\{\textsc{AND}, \textsc{OR}\}$ binary functions and $L = \{x_1, \ldots, x_n\}$
input variables. When the program quality is evaluated using the complete truth
table, the fitness function $f(X)$ counts the number of truth-value assignments
(or \emph{inputs}) on which the output of the candidate solution $X$ differs
from that of the target function $\hh(\mathbf{x})$ (i.e., $\andn (\mathbf{x})= x_1 \wedge
\ldots \wedge x_n$ and $\orn (\mathbf{x}) = x_1 \vee \ldots \vee x_n$ for the \andn and \orn
problems respectively). As observed by \citet{LissovoiO19}, a conjunction of
$a$ distinct variables differs from \andn on $2^{n-a}-1$ truth-value
assignments.

We will analyse the performance of the \RLSGP algorithm, which repeatedly
chooses the best between its current solution and an offspring generated by
applying a tree mutation operator. In addition to considering the classic
HVL-Prime mutation operator, which with equal probability inserts, deletes, or
substitutes a leaf node in the current solution \citep{DurrettNO11}, we also
propose a slightly modified version (\emph{HVL-Prime with subtree deletion})
that is able to delete arbitrary subtrees in a single mutation.

We observe that the presence of disjunctions in the current solution may lead
to bloat issues. Each OR increases the minimum number of leaf nodes required to
represent the exact conjunction, and diminishes the effect of insertions
beneath it on the overall program semantics. Additionally, \peter{since} the classic
HVL-Prime \peter{operator} performs deletions by removing a leaf node and its immediate parent,
it may be difficult for it to remove disjunctions placed high up in the tree.
To counteract this, we add a simple bloat control mechanism to \RLSGP, making
it reject trees which contain more than $\ell$ leaf nodes, as described in
Algorithm~\ref{alg:RLS-GP}. 

\begin{algorithm}[t]
	\begin{algorithmic}[1]
		\State Initialise an empty tree $X$
		\For{$t \gets 1, 2, \ldots$}
			\State $X' \gets \text{HVL\text-Prime}(X)$
			\If{$\text{LeafCount}(X') \leq \Tmax$ \textbf{and} $f(X') \leq f(X)$}
				\State $X \gets X'$
			\EndIf
		\EndFor
	\end{algorithmic}
	\caption{The \RLSGP algorithm with a tree size limit $\Tmax$.} \label{alg:RLS-GP}
\end{algorithm}

With the tree size limit $\ell$ in place, applying the original HVL-Prime 
mutation operator \citet{DurrettNO11} may cause \RLSGP to get stuck on
a local optimum.

\begin{theorem} \label{thm:rls-stuck}
The expected optimisation time of \RLSGP with leaf-only
deletion and substitution sub-operations of HVL-Prime,
and any $\Tmax > 0$ on \andn or \orn with $F=\{\text{AND}, \text{OR}\}$ is
infinite.
\end{theorem}
\begin{proof}
\RLSGP may construct trees which contain $\Tmax$ leaf nodes and cannot be
further improved by local mutations.

For \andn, consider a tree constructed by initially creating a disjunction of
$\Tmax/2$ $x_1$ leaf nodes, and then transforming each $x_1$ leaf into an $x_1
\wedge x_2$ subtree. No leaf node in the final tree can be deleted or
substituted without decreasing fitness, and no insertion will be accepted due
to the tree size limit, rendering \RLSGP unable to reach the optimum. As this
tree can be constructed with non-zero probability, the expected time to
construct the optimal solution is infinite by the law of total expectation.

For \orn, an equivalent issue arises when an initial conjunction of $\Tmax/2$
$x_1$ leaf nodes has each leaf transformed into an $x_1 \vee x_2$ subtree.
\end{proof}

To avoid this issue, we modify the deletion operation of HVL-Prime to allow
deletion of subtrees as described in Algorithm~\ref{alg:HVLPrime}. The only difference with the original HVL-Prime operator is that in line 9 only choosing a leaf node for deletion was allowed while now any node may be chosen for deletion.

\begin{algorithm}[t]
	\begin{algorithmic}[1]
		\State \textbf{Inputs: } a tree $X$, set of available literals $L$, set of available binary functions $F$.
		\State Choose $op \in \{\text{INS}, \text{DEL}, \text{SUB}\}$, $l \in L$, $f \in F$ uniformly at random
		\If{$X$ is an empty tree}
			\State Set $l$ to be the root of $X$.
		\ElsIf{$op = \text{INS}$}
			\State Choose a node $x \in X$ uniformly at random
			\State Replace $x$ with $f$, setting the children of $f$ to be $x$ and $l$, order chosen u.a.r.
		\ElsIf{$op = \text{DEL}$} \Comment{modified (subtree) deletion}
			\State Choose a node $x \in X$ uniformly at random
			\State Replace $x$'s parent in $X$ with $x$'s sibling in $X$
		\ElsIf{$op = \text{SUB}$}
			\State Choose a leaf node $x \in X$ uniformly at random
			\State Replace $x$ with $l$.
		\EndIf
		\State \Return the modified tree $X$
	\end{algorithmic}
	\caption{HVL-Prime with subtree deletion on tree $X$.} \label{alg:HVLPrime}
\end{algorithm}

We use the term \textit{sampled error} to refer to the fitness value of a particular solution
in a particular iteration, and \textit{generalisation error} to refer to the probability that
a particular solution is wrong on an input chosen uniformly at random from the set of all
$2^n$ possible inputs. When the program quality is evaluated using the complete truth table,
the sampled error of a solution is always exactly $2^n$ times its generalisation error.
When the complete truth table is used, the goal of the GP system is to construct a solution that is
semantically equivalent to the target function, i.e.\ achieve a sampled (and generalisation)
error of 0.

As it is computationally infeasible to evaluate all $2^n$ possible inputs for large
values of $n$, we also analyse the behaviour of \RLSGP when evaluating the solution quality
based on $s \in \mathrm{poly}(n)$ inputs chosen uniformly at random from the set of
all possible inputs. A fresh set of $s$ inputs is chosen in each iteration, and
$f(X)$, or the \textit{sampled error},
then refers to the number of inputs, among the chosen $s$, on which $X$
differs from the target function. The sampled error is thus a random variable,
and its expectation is exactly $s$ times the generalisation error of the solution. We bound
the probability that the sampled error deviates from its expectation in
Lemma~\ref{lem:sampled-error} below.
When a \peter{polynomially sized} training set is used to evaluate the program quality,
the goal of the GP system is to construct a solution with a low generalisation error.
On \andn, and most other non-trivial problems, we do not expect the GP systems to
reach a generalisation error of 0 while $s$ remains polynomial with respect to the
problem size, unless the problem's fitness landscape is well understood and a
problem-specific training set is used \citep{LissovoiO19}. We assume that this is not the 
case, and that the aim is to find a solution that has a\peter{n arbitrarily small polynomial} generalisation error.
Note that we use the following notation throughout the paper: $\N := \{0,1,2,
\dots\}$, $\lg(n)$ and $\ln(n)$ denoting the base 2 and the natural logarithms 
of $n$, while $\log n$ is used in asymptotic bounds.

\begin{lemma} \label{lem:sampled-error}
Let $s \in \mathrm{poly}(n)$ be the number of inputs sampled by the GP system,
$G$ be the generalisation error of a solution, and $X$ be the random variable
that denotes the sampled error of that solution. Then, for any $c$ that is at
least a positive constant,
$$ |G s - X| \le \max\{c \lg n, G s \} $$
with probability at least $1 - n^{-\Omega(c)}$.
\end{lemma}

\begin{proof}
$X$ is a sum of $s$ Bernoulli variables, each
with a probability $G$ of assuming the value $1$ (and $0$ otherwise), and hence
$E[X] = Gs$. As both $X$ and $Gs$ are non-negative, $G s - X \leq G s$,
and we focus solely on the case where $X$ significantly exceeds its expectation,
the probability of which can be bounded by applying a Chernoff bound.

Suppose that $E[X] \geq (c/2) \lg n$; then, $\Pr[X \geq (1+1) E[X]] \leq e^{-E[X]/3} \leq n^{-\Omega(c)}$; and hence $|Gs - X| < Gs$, with probability at least $1-n^{-\Omega(c)}$. Otherwise, we upper bound $E[X] \leq \mu^+ = (c/2) \lg n$, and
apply a Chernoff bound using $\mu^+$ \citet[Theorem~1.10.21]{Doerr20bookchapter}, obtaining $\Pr[X \geq (1+1) \mu^+] \leq e^{-\mu^+/3} = n^{-\Omega(c)}$; and hence $|Gs - X| \leq X \leq c \lg n$ with probability at least $1-n^{-\Omega(c)}$.
\end{proof}

\section{Complete truth table} \label{sec:CTT}
In this section, we will present a runtime analysis of the \RLSGP algorithm
with subtree deletion (i.e.\ Algorithm~\ref{alg:RLS-GP} using Algorithm~\ref{alg:HVLPrime} as the mutation operator)
on the \andn problem, using the complete truth table to evaluate the solution
quality, i.e., executing each constructed program on all $2^n$ possible inputs.

\begin{theorem} \label{thm:lower-and}
The expected runtime of \RLSGP with $F=\{AND,$ $OR\}$, $L=\{x_1, \ldots,
x_n\}$, and $\Tmax \geq n$ on \andn or \orn is $E[T] = \Omega(n \log n)$.
\end{theorem}
\begin{proof}
No tree which does not contain all $n$ distinct variables can be
equivalent to the \andn or \orn functions.
By a standard coupon collector argument, $\Omega(n \log n)$ insertion or
substitution operations are required to insert all $n$ distinct variables into
the tree.
\end{proof}

The following drift theorem requires that the expected progress when at distance $d$ from the target is of order $\Omega(d \log d)$. This assumption is slightly stronger than the linear (i.e.\ $\Omega(d)$) progress assumed in the multiplicative drift theorem. Despite this apparently small difference, the resulting bounds for the expected time to reach the target differ drastically. For an initial distance of $d_0$, they are, roughly speaking, $O(\log d_0)$ for the multiplicative drift situation and $O(\log \log d_0)$ for our super-multiplicative drift.

\begin{theorem}[Super-multiplicative drift theorem]\label{thm:smdrift}
  Let $\gamma > 1$ and $\delta > 0$.  Let $X_0, X_1, \dots$ be random variables taking values in $\Omega = \{0\} \cup [1,\infty)$. Assume that for all $t \in \N$ and all $x \in \Omega \setminus \{0\}$ such that $\Pr[X_t = x] > 0$ we have 
  \begin{equation}
  E[X_t - X_{t+1} \mid X_t = x] \ge (\log_\gamma(x)+1) \delta x. \label{eq:superdrift}
  \end{equation}
  Then the first hitting time $T = \min\{t \in \N \mid X_t = 0\}$ of zero satisfies \[E[T \mid X_0] \le \frac 3 \delta + \frac{2 (2+\log_2 \log_\gamma \max\{\gamma,X_0\}) \ln \gamma}{\delta}
.\]
\end{theorem}

\begin{proof}
  For all $k \in \N_{\ge 1}$, let $T_k := \min\{t \in \N \mid X_t < \gamma^{2^{k-1}}\}$. We first show that 
  \begin{equation*}
  E[T_{k} - T_{k+1}] \le \frac{1+2^k \ln \gamma}{(2^{k-1}+1) \delta}
  \end{equation*}
  holds for all $k \ge 1$. To this aim, we regard the process $Y_t$ defined for all $t \in \N$ by $Y_t = X_t$ if $t \le T_k-1$ and $Y_t = 0$ otherwise. By definition, $T^Y_k := \min\{t \in \N \mid Y_t < \gamma^{2^{k-1}}\}$ satisfies $T^Y_{k} = T_{k}$. We show that the process $(Y_t)$ satisfies the multiplicative drift condition,
  \[E[Y_t - Y_{t+1} \mid Y_t] \ge (2^{k-1}+1) \delta Y_t.\]
  If $t \ge T_k$, then both $Y_t = 0$ and $Y_{t+1} = 0$. Consequently, the multiplicative drift condition is trivially satisfied. In the more interesting case that $t < T_k$, we have $Y_t \ge \gamma^{2^{k-1}}$ and $Y_t = X_t$. From this, $Y_{t+1} \le X_{t+1}$, and~\eqref{eq:superdrift}, we conclude 
	\[E[Y_t - Y_{t+1} \mid Y_t] \ge E[X_t - X_{t+1} \mid X_t] \ge (\log_\gamma(X_t)+1) \delta X_t \ge (2^{k-1} + 1) Y_t,\]
	again showing the multiplicative drift condition.
  
  Let $T^Y := \min\{t \in \N \mid Y_t = 0\}$. Since $T^Y = T^Y_k = T_k$ and since $Y_t \le \gamma^{2^k}$ for all $t \ge T_{k+1}$, the multiplicative drift theorem \citep{DoerrJW12} yields $E[T_{k} - T_{k+1}] = E[T^Y - T^Y_{k+1}] \le \frac{1+\ln \gamma^{2^k}}{(2^{k-1}+1)\delta} = \frac{1+2^k \ln \gamma}{(2^{k-1}+1)\delta}$. 
  
  By a simple application of the multiplicative drift theorem, we also observe that $E[T - T_1] \le \frac{1+\ln \gamma}{\delta}$.
  
  In the following, we condition on the initial value $X_0$. Assume that $X_0 \in [\gamma^{2^{k-1}},\gamma^{2^{k}})$ for some $k \in \N_{\ge 1}$. Then $T_{k+1} = 0$ and thus $T = \sum_{i=1}^k (T_i - T_{i+1}) + (T - T_1)$. We compute
  \begin{align*}
  E[T] & = \sum_{i=1}^k E[T_i - T_{i+1}] + E[T - T_1]
  \le \sum_{i=1}^k \frac{1+2^i \ln \gamma}{(2^{i-1}+1) \delta} + \frac{1+\ln\gamma}{\delta}\\
  &\le \frac 3 \delta + \frac{2(k+1)\ln \gamma}{\delta}
  \le \frac 3 \delta + \frac{2 (2+\log_2 \log_\gamma X_0) \ln \gamma}{\delta}.
  \end{align*}
  For $X_0 < \gamma$, we have in an analogous way $E[T] \le \frac{1 + \ln(X_0)}{\delta} \le \frac{1 + \ln \gamma}{\delta}$. This proves the claim.
\end{proof} 

The proof of Theorem~\ref{thm:smdrift} estimates the super-multiplicative drift
by piece-wise multiplicative drifts. We preferred this proof method because of
its simplicity and because it could, by using the multiplicative drift theorem
with tail-bounds \citep{DoerrG13algo}, also lead to tail-bounds for
super-multiplicative drift as well\footnote{For this, we note that the multiplicative drift theorem with tail bounds~\citep{DoerrG13algo} shows that the phase lengths $T_k - T_{k-1}$ and $T - T_1$ are stochastically dominated (see, e.g.,~\cite{Doerr19tcs}) by their expectation plus a geometrically distributed random variable. Consequently, the deviation of $T$ above its expectation is stochastically dominated by a sum of independent geometrically distributed random variables and thus can be bounded by tail bounds for such sums~\cite[Section 1.10.4]{Doerr20bookchapter}. We omit further details since in this work we do not need such tail bounds.}. An alternative approach which would improve the time bound
by a constant factor (again a feature we are not interested in here) would be
to use variable drift \citep{MitavskiyRC09,Johannsen10}.

We use the super-multiplicative drift theorem to prove our upper bound on the
expected runtime of \RLSGP when using the complete truth table as the training
set. We initially focus on the \andn problem, and will later show that because
of the symmetry in the \orn and the \andn problems, equivalent results also
hold for \orn. We begin by bounding the
time spent in iterations in which the tree is not full, i.e.\ it has not
reached the size limit of having $\ell$ leaf nodes.

\begin{lemma} \label{lem:unfull-runtime}
Consider a run of \RLSGP on \andn, with $F=\{AND,$ $OR\}$, $L=\{x_1, \ldots, x_n\}$, and a tree size limit of $\Tmax \geq n$.
Let $T$ be the number of iterations before the optimum is found, and $T_0 \leq
T$ be the number of these iterations in which the parent individual is not a
full tree. Then, $E[T_0]=O(\Tmax \, n \log^2 n)$.
\end{lemma}

\begin{proof}
To bound $E[T_0]$, we will apply Theorem~\ref{thm:smdrift} using the solution
fitness as the potential function, and considering only the iterations in which
the tree is not full. While the tree \emph{is} full, we instead rely on the
elitism of the \RLSGP algorithm to not accept mutations which increase the
potential function value (i.e., offspring with a worse fitness value). Thus,
the $T_0$ iterations in which the tree is not full need not be contiguous.

In an iteration starting with a tree containing less than $\Tmax$ leaf nodes, it is possible to insert a new leaf node $x_i$ with an $AND$ parent anchored at the root of the tree.  We call such an operation a root-and. The probability that in one iteration a root-and with a fixed variable $x_i$ is performed, is at least $\frac 13 \cdot \frac 12 \cdot \frac 1 {2\ell} \cdot \frac 1n = \frac{1}{12\ell n}$.

We compute the expected fitness gain caused by such modifications. Because the fitness never worsens, it suffices to regard certain operations that improve the fitness. Recall further that the fitness is just the number of assignments to the variables $x_1, \dots, x_n$ such that the tree evaluates differently from \andn.

Let $x_1, \dots, x_n$ be such an assignment. This implies that not all $x_i$ are true, because any tree generated by \RLSGP evaluates correctly to true for the all-true assignment. Assume that exactly $k \ge 1$ of the variables $x_1, \dots, x_n$ are false, but that our tree solution evaluates to true. Then there are exactly $k$ variables such that a root-and with one of them would make the program evaluate to false on this assignment (and thus improve the fitness, since false is the correct output). The probability for such a mutation is at least $\frac{k}{12 \ell n}$.

For any $1 \leq i \leq n$, there are exactly $\binom{n}{i}$ inputs where
exactly $i$ variables are set to false, and exactly
$\sum_{i=1}^{k-1} \binom{n}{i}$ inputs where less than $k$
variables are set to false. Thus, if the fitness of the current solution is at
least $M_k = 2 \sum_{i=1}^{k-1} \binom{n}{i}$, at least half of the inputs
contributing to the fitness have at least $k$ variables set to false.

By only regarding the progress caused by these, we have, for $x \ge M_k$,
\begin{gather}
E\left[f(X^t) -f(X^{t+1}) \mid f(X^{t}) = x\right] \geq \frac{1}{12\Tmax} \, \frac{k}{n} \, x. \label{eq:k-drift}
\end{gather}

Since for $n$ sufficiently large we have $M_k \le 2 n^{k-1}$ for all $k \in [1..n]$. This implies that for all $x \in [1..2^n]$ and all $t \in \N$, we have
\begin{align*}
E\left[f(X^t) -f(X^{t+1}) \mid f(X^{t}) = x\right] 
&\ge \tfrac{1}{12 \ell n} \left(\lfloor \log_n(x/2) \rfloor+ 1\right) x \\
&\ge \tfrac{1}{36 \ell n} \left(\log_n(x)+1\right) x,
\end{align*}
where the last estimate uses $n \ge 2$. Hence, Theorem~\ref{thm:smdrift} with $\gamma=n$ and $\delta = 1/36\ell n$ gives \[E[T] \le 36 \ell n (3+ 2(2 + \log_2 \log_n 2^n))\ln n) = O(\ell n \log^2 n).\]
\end{proof}

To prove following theorem, we will show that with high probability,
the parent solution contains fewer than $\ell$ leaf nodes in at least a
constant fraction of any $t \in \Omega(\Tmax\, n \log^2n)$ iterations.
Intuitively, this means that the conditions of Lemma~\ref{lem:unfull-runtime}
apply often enough to not affect the asymptotic expected runtime (i.e.\ $E[T] = O(E[T_0])$).

\begin{theorem} \label{thm:CTT-runtime}
Consider a run of \RLSGP on \andn, using $F=\{AND,$ $OR\}$, $L=\{x_1,\ldots,x_n\}$, and a tree size limit of $\Tmax = (1+c) n$.
Let $T$ be the number of iterations before the optimum is found. If
$c = \Theta(1)$, then $E[T] = O(\Tmax \, n \log^2 n)$.
\end{theorem}
\begin{proof}
Let $T' = c^* \Tmax n \, \log^2 n$, for some constant $c^* > 0$, be an upper
bound on the expected number of iterations $E[T_0]$ in which the tree is not 
full
before the optimum solution is found per Lemma~\ref{lem:unfull-runtime}.
By an application of Markov's inequality, the
probability that the optimum is found in at most $2T'$ such iterations is at
least $1/2$. We will show that if $\Tmax = (1+c) n$, for any constant $c > 0$,
$2T'$ such iterations occur in $(2+c')T'$ iterations with high
probability, where $c' > 0$ is constant with respect to $n$. The theorem
statement then follows from a simple waiting time argument: during each period
of $(2+c')T'$ iterations, the optimum is found with probability at least $1/2
\cdot (1 - o(1)) = \Omega(1)$, so the expected number of such periods before the
optimum is found is at most $O(1)$, and thus the expected runtime is at most
$O(T') = O(\Tmax n \, \log^2 n)$ iterations.

We will now show that during any $N \in \Omega(\Tmax\, n \log^2n)$ iterations,
with high probability and for some constant $c'' > 0$, deletions of at least
$c'' N$ leaf nodes in total will be accepted. 
As each iteration can at most increase the number of leaf nodes in the tree
by 1, there will with high probability be
at least $c'' N$ iterations in which the tree is not full among any $(1+c'') N$
iterations. As $T' \in \Omega(\Tmax\, n \log^2n)$, $2T'$ iterations in which
the tree is not full will with high probability occur in $(2+c')T'$ iterations
where $c' = 2/c'' = \Omega(1)$.

Consider a tree $X$ with exactly $\Tmax$ leaf nodes. Let $L_A(X)$ be a set of
leaf nodes connected to the root of $X$ via only AND nodes, and call
\emph{essential} all the leaf nodes in this set that contain a variable which
only appears on nodes in this set exactly once. If $X$ is non-optimal, at most
$n-1$ leaf nodes in $X$ are essential, and at least $\Tmax - (n-1)$ leaf nodes
are non-essential. All
non-essential nodes are either directly deletable (in the case of redundant
copies of variables in $L_A(X)$), or indirectly deletable (by deleting a branch
at any of their OR ancestors).

Every non-essential leaf node can thus be deleted by performing an HVL-Prime
deletion sub-operation on at least one node in the tree.
For some non-essential leaf nodes, a larger subtree may need to be deleted to
remove the leaf without adversely impacting fitness. The longer waiting time
for such subtree deletions (requiring that the root of the subtree be chosen
for deletion rather than one of the many leaf nodes in the subtree) is balanced
by the increased number of leaf nodes deleted as part of the mutation.
We note that the tree contains $2\Tmax-1$ nodes, and thus for $\Tmax \geq
(1+c)n$ and any $c > 0$, an HVL-Prime mutation in expectation reduces the
number of leaf nodes in the tree by at least 
$$\frac{1}{3}\frac{\Tmax - (n - 1)}{2\Tmax - 1} \geq
\frac{\Tmax - n}{6\Tmax} \geq \frac{c}{6+6c} \geq \delta \in \Omega(1),$$
where $\delta > 0$ is a positive constant, as $c \in \Omega(1)$.

Let $X_1, \ldots, X_N$ be the number of leaf nodes deleted in an accepted
mutation during each iteration performed while the tree is full, and $X =
\sum_{i=1}^{N} X_i$. Furthermore, define a sequence $Z_0, \ldots, Z_N$, where
$Z_0 := 0$ and $Z_i := Z_{i-1} + X_i - \delta$; clearly, $Z_N - Z_0 = Z_N = X -
\delta N$. We will show that $Z_N > - \delta N/2$ (and therefore $X >
\delta N/2 \in \Omega(N)$) holds with high probability.

As $E[Z_{i} \mid Z_1, \ldots Z_{i-1}] = Z_{i-1} + E[X_{i} \mid Z_1, \ldots
Z_{i-1}] - \delta \geq Z_{i-1}$, the sequence $Z_0, \ldots, Z_N$ is a
sub-martingale, and $c_i := |Z_i - Z_{i-1}| \leq \Tmax$. Hence, by applying the
Azuma-Hoeffding inequality for $N \in \Omega(\Tmax\, n \log^2n)$ and $t =
\delta N/2$,
\begin{align*}
\Pr[Z_N - Z_0\leq - t] & \leq \exp\left(\frac{-t^2}{2\sum_{i=1}^{N} c_i^2}\right)
\leq \exp\left(\frac{-\delta^2 N}{8 \Tmax^2}\right)
\leq n^{-\Omega(\log n)}
\end{align*}
as $N/\Tmax^2 = \Omega(n \Tmax \log^2 n / \Tmax^2) = \Omega(\log^2 n)$ for
$\Tmax = (1+c)n$ where $c$ is a constant.

Thus, there exists a constant $c'' > 0$ such that over the course of $N \in
\Omega(n \Tmax \log^2 n)$ iterations where the tree is full, deletions of at
least $\delta N/2 = c'' N$ leaf nodes are accepted with high
probability, and hence over the course of $(2/c'') N$ iterations, at least
$2N$ iterations occur while the tree is not full with high probability.
Setting $N = T' = c^* n \Tmax \log^2 n$ iterations per
Lemma~\ref{lem:unfull-runtime} completes the proof: among $\Theta(T')$
iterations, at least $\Omega(T')$ will take place while the tree is not full,
allowing the application of the Markov inequality and waiting time arguments to
produce the bound on the expected runtime.
\end{proof}

\newmarker Additionally, \peter{we prove that} the \RLSGP algorithm using the same function and
terminal sets is able to evolve disjunctions of $n$ variables efficiently when
using the complete truth table to evaluate the solution fitness.
Theorem~\ref{thm:CTT-runtime-OR} formalises this result, and is proven by the
same methods as Theorem~\ref{thm:CTT-runtime}, taking advantage of the symmetry
between the search spaces of \peter{the} \andn and \orn \peter{problems}.

\begin{theorem} \label{thm:CTT-runtime-OR}
Consider a run of \RLSGP on \orn, using $F=\{AND,$ $OR\}$, $L=\{x_1,\ldots,x_n\}$, and a tree size limit of $\Tmax = (1+c) n$.
Let $T$ be the number of iterations before the optimum is found. If
$c = \Theta(1)$, then $E[T] = O(\Tmax \, n \log^2 n)$.
\end{theorem}
\begin{proof}
We note that there is a symmetry between the output vectors of the \andn and
\orn target functions: while \andn returns true only when all $n$ input
variables are true, \orn returns false only when all $n$ input variables are
false. Similarly, a disjunction of $a$ distinct variables is wrong on
$2^{n-a}-1$ inputs on \orn.

Lemma~\ref{lem:unfull-runtime} carries over to the \orn problem by
considering the time taken to find the optimal solution through insertions of
\emph{disjunctions} at the top of the tree, observing that for any $1 \leq i
\leq n$, there are exactly $\binom{n}{i}$ inputs where exactly $i$ variables
are set to \emph{true}, and exactly $\sum_{i=1}^{k-1} \binom{n}{i}$ inputs
where less than $k$ variables are set to \emph{true}. Thus, if the fitness of
the current solution is at least $M_k = 2 \sum_{i=1}^{k-1} \binom{n}{i}$, at
least half of the inputs contributing to the fitness have at least $k$
variables set to \emph{true}, allowing the error to be at least halved by
inserting one of these $k$ variables.

The proof of Theorem~\ref{thm:CTT-runtime} then carries over to the \orn
problem by adjusting the definition of essential leaf nodes: let $L_A(X)$ be a
set of leaf nodes connected to the root of $X$ via only \emph{OR} nodes, and
call \emph{essential} all the leaf nodes in this set that contain a variable
which only appears on nodes in this set exactly once. A full tree can then be
shrunk by removing non-essential nodes, either directly (in the case of
redundant copies of variables in $L_A(X)$) or by deleting a branch at any of
their AND ancestors.

Thus, the runtime bound derived in the proof of Theorem~\ref{thm:CTT-runtime}
also holds for the \orn problem.
\end{proof}

If the terminal set is extended by adding negations of all positive input
variables, \RLSGP with $F=\{AND, OR\}$ is able to represent any Boolean
function (by having a tree expressing its disjunctive or conjunctive normal
form, where any negations are pushed all the way down to individual variables).
For the \andn problem with $F=\{AND\}$, \cite{LissovoiO19} have shown that such
a terminal set extension renders \RLSGP inefficient when using the complete
training set. In that setting, any solution containing a contradiction
(e.g., $x_1 \wedge \overline{x}_1$) will have the next-to-optimal fitness
(being wrong on only the all-true input). This deprives \RLSGP of a fitness
signal until an optimal solution is constructed by random chance inserting all
positive terminals into the tree and removing copies of all negative terminals
from the tree. \peter{Since} all mutations that do not remove the last-remaining
contradiction are accepted, this requires at least exponential time in
expectation.

We \peter{conjecture} that a similar result would hold for \RLSGP with $F=\{AND, OR\}$, \peter{ and subtree deletion} on
both the \andn and \orn problems, as trees containing a contradiction on \andn,
or a tautology on \orn, can be constructed, and would only be wrong on a single
input. However, 
\peter{the negative drift analysis proof technique used in \cite{LissovoiO19} cannot be easily applied with the subtree mutation operator
as large jumps in the search space may occur with high probability. Thus we leave this conjecture for future work, and focus in the next section on the positive result 
for the realistic scenarios where training sets of polynomial size are used.}

\section{Polynomially sized training sets} \label{sec:Incomplete}

\editmarker
While the previous section provides polynomial bounds on the number of
iterations required to evolve a conjunction or a disjunction of all $n$
variables, calculating the solution quality by comparing the programs' output to
that of the target function on all of the $2^n$ possible inputs in each
iteration requires exponential computational effort. Thus, it is only
computationally feasible for \peter{evolving Boolean functions of 
modest size.} 

In this section, we consider the behaviour of the \RLSGP algorithm using only
a polynomial computational effort in each iteration. To this end, the solution
quality is compared by evaluating the output of the parent and offspring 
solutions and the target function on only a polynomial number of inputs (the ``training set'')
sampled uniformly at random from the set of all possible inputs in each iteration.
This setting was previously considered by \citet{LissovoiO19}, where it was shown that
\RLSGP and (1+1)~GP using $F = \{AND\}$ are able to construct a solution
with $O(\log n)$ distinct variables which fits a random polynomially large training set in expected $O(\log n)$ time.

For our main theoretical result below, we opt to have \RLSGP terminate and return
a solution once the sampled error on the training set is below a logarithmic acceptance
threshold. This effectively prevents \RLSGP from entering a region of the search
space where the mechanism it uses to evaluate the program quality is overly noisy. This
slightly decreases the expected solution quality, but 
still allows to guarantee arbitrarily small generalisation errors. 

\begin{theorem} \label{thm:incomplete-single-run} \editmarker
For any constant $c > 0$, consider an instance of the
\RLSGP algorithm with $F = \{AND, OR\}$, $L = \{x_1, \ldots, x_n\}$, $\Tmax \geq n$,
using a training set of $s = n^c \lg^2 n$ rows sampled uniformly
at random from the complete truth table in each iteration to evaluate the solution
quality, and terminating when the sampled error of the solution is at most
$c' \lg n$, where $c'$ is an appropriately large constant. For both the
\andn and \orn problems, with probability at least $1 - O(\log^2(n)/n)$, the algorithm will terminate within
$O(\log n)$ iterations, and return a solution with a generalisation
error of at most $n^{-c}$.
\end{theorem}
To prove this theorem, we will show that \RLSGP is able to create a tree that
contains no more than one copy of each variable, no undesired functions 
(i.e., OR and AND on the \andn and \orn problems respectively), and enough
distinct variables to sample an error below the acceptance threshold within
$O(\log n)$ iterations with probability at least $1 - O(\log^2(n)/n)$.
Additionally, we will show that with high probability the GP system will not
terminate early (i.e.\ it will not return a solution with a generalisation error greater 
than $n^{-c}$ \peter{for an arbitrary constant $c$}).

\begin{lemma} \label{lem:incomplete-happy-runtime} \editmarker
If, in the setting of Theorem~\ref{thm:incomplete-single-run},
\RLSGP never accepts solutions containing undesired function nodes 
(i.e., OR and AND on the \andn and \orn problems respectively)
or multiple copies of
any variable, and never accepts solutions with a worse generalisation error than
their ancestors, it will within $O(\log n)$ iterations reach a solution with a
sampled error below $c' \lg n$, where $c' > 0$ is an appropriate constant,
with probability at least $1 - O(1/n)$.
\end{lemma}
\begin{proof}
To ensure that an error below $c' \lg n$ is sampled, we consider the time required
to construct a solution with an expected sampling error of at most $(c'/4) \lg n$.
Such a sampling error can be achieved by a generalisation error of
at most $((c'/4) \lg n) / (n^c \lg^2 n) = (c'/4) n^{-c} / \lg n \geq n^{-(c+1)}$
(for a sufficiently large $n$), i.e.\ a conjunction \edithilight{(on the \andn problem, or a disjunction on the \orn problem)}
of $(c+1)\lg n$ variables or more.

The time required to construct such a solution under the lemma's conditions
can be bounded by lower-bounding the probability of inserting a new variable connected
to the tree using an appropriate function node (i.e., AND on \andn and OR on \orn), and using a Chernoff bound to show that a sufficient
number of such insertions occur within a particular number of iterations (as the
number of distinct variables in the current solution is never reduced by the lemma's
conditions). Specifically, suppose that the current solution contains $i < n/2$ 
distinct variables and no undesired function nodes, and let $X_i$ be the event that a mutation inserts
a new variable and connects it to the tree using an appropriate function node, and is accepted. We bound
$ \Pr[X_i] \geq (1/3) (1/2) (n-i)/n \geq \delta $, i.e.\
$\delta \geq 1/12$ for $i < n/2$. The probability that
at least $(c+1)\lg n$ such mutations are accepted within $(c''/\delta) (c+1) \lg n = O(\log n)$
iterations is then, by applying a multiplicative Chernoff bound, at
least $1 - e^{-\Omega(c'' \log n)} = 1 - n^{-\Omega(c'')}$.
Thus, when $c''$ is a sufficiently large constant, this probability is at least
$1 - O(1/n)$.

We bound the probability that a solution with a low-enough expected sampled error
does not meet the acceptance threshold by applying Lemma~\ref{lem:sampled-error}.
Once a solution with an expected sampled error of at most $(c'/4) \lg n$ is 
constructed, the probability that its sampled error exceeds the acceptance threshold
is at most $n^{-\Omega(c')}$. Thus, when $c'$ is picked appropriately, the solution
is accepted immediately with probability at least $1-O(1/n)$.

By combining the failure probabilities using a union bound, we conclude that
\RLSGP under the conditions of the lemma and with an appropriately-chosen
constant $c'$, is able to construct a
solution with an acceptable sampled error within $O(\log n)$ iterations with
probability at least $1-O(1/n)$.
\end{proof}

We will now use this bound on the runtime of \RLSGP to show that it is likely
to avoid all of the potential pitfalls preventing the application of
Lemma~\ref{lem:incomplete-happy-runtime}.

\begin{lemma} \label{lem:incomplete-no-bad-things} \editmarker
In the setting of Theorem~\ref{thm:incomplete-single-run},
with probability at least $1-O(\log^2(n)/n)$, during the first $O(\log n)$
iterations and while the expected sampled error of its current solution remains
above $(c'/4)\lg n$, \RLSGP is able to avoid accepting mutations which: (1)
insert copies of a variable already present in the current solution, (2) insert
undesired function nodes (i.e., OR and AND on the \andn and \orn problems respectively),
or (3) increase the generalisation error of the current solution.
\end{lemma}
\begin{proof}
For claim (1), we note that within the first $O(\log n)$ iterations, the tree
will contain at most $O(\log n)$ distinct variables (as each iteration of 
\RLSGP is only able to insert one additional variable). Thus, the probability
that a mutation operation adds a variable which is already present in the
solution (using either the insertion or substitution sub-operation of
HVL-Prime) is at most $O(\log n / n)$.
By a union bound, this does not occur during the first $O(\log n)$
iterations with probability at least $1 - O(\log^2 (n) / n)$.

For claim (2), we note that there are two main ways an undesired function node can be introduced into
the solution by an insertion operation. First, the function can be semantically neutral, e.g.
on the \andn problem, if the ancestor contains only ANDs and unique variables,
a leaf $x_i$ could be replaced with $x_i \vee x_i$ without affecting any of the solution's outputs.
Second, the sampling process used to evaluate the solution
fitness may not have sampled any inputs on which the offspring is wrong and the ancestor
is correct. We will consider the two possibilities separately.

As semantically-neutral insertions of undesired function nodes require inserting a
duplicate copy of a variable, claim (1) already provides the desired probability
bound on these insertions not occurring within $O(\log n)$ iterations
(and hence not being accepted).
All other undesired function node insertions will increase the generalisation error of the solution. The
magnitude of this increase depends on the number of distinct
variables in the subtree displaced by the insertion, with insertions displacing
only a single leaf node being the easiest to accept.

If a leaf of the ancestor solution is displaced by an insertion which adds an
undesired function node and a new variable, we use the term \textit{witness} to
refer to inputs on which the constructed offspring solution produces incorrect
output while the parent solution produces correct output. For the \andn
problem, \emph{witness} inputs are those where the displaced variable to false
while setting the remaining variables in the offspring solution to true, while
on the \orn problem, the displaced variable is set to true while the remaining
variables in the offspring solution are set to false.
Witness inputs demonstrate that the offspring solution is worse than its
parent; on all other inputs, the parent and offspring produce the same output. 
Thus, as long as
the sampling procedure finds at least one witness, \RLSGP will reject the mutated
solution. 

Suppose the ancestor solution contains $U$ distinct variables (and uses
only the desired function nodes, i.e., AND on \andn and OR on \orn);
it is then incorrect on $2^{n-U}-1$ possible inputs, while there are at least
$2^{n-(U+1)}$ witnesses; i.e., the probability of randomly selecting a witness
is at least half that of randomly selecting an input on which the ancestor is wrong.
Thus, if the expected sampled error of the ancestor solution is at least $X$, 
the expected number of witnesses in the sample is at least $X/2$. By a Chernoff bound,
the probability that fewer than $(c'/16) \lg n$ witnesses are present in the sample is at most
$e^{-(c'/128) \lg n} = n^{-\Omega(c')}$. By setting the constant $c'$
appropriately, this probability can be made into $O(1/n)$.
Hence, by a union bound,
the probability that no insertion of an undesired function node which increases the generalisation error is accepted
within $O(\log n)$ iterations while the expected sampled error of the solution remains
above $(c'/4) \lg n$ is at least $1-O(\log(n)/n)$.

For claim (3), we note that decreasing the number of distinct variables
in the solution more than doubles its generalisation error. Applying a similar
argument as for rejecting detrimental undesired function node insertions above
(with at least $X$ witnesses expected in a sampled training set),
the probability that no mutations increasing the generalisation error are accepted
during $O(\log n)$ iterations is at least $1-O(\log(n)/n)$.

Combining the error probabilities of the three claims using a union bound yields
the theorem statement.
\end{proof}

Finally, we show that with high probability, \RLSGP does not
terminate unacceptably early (i.e., by sampling an error below the acceptance
threshold for a solution with a worse generalisation error than desired by
Theorem~\ref{thm:incomplete-single-run}).

\begin{lemma} \label{lem:no-bad-solutions}
In the setting of Theorem~\ref{thm:incomplete-single-run},
with high probability, no solution with a generalisation error greater than
$n^{-c}$ has a sampled error of at most $c' \lg n$ on a set of $s \geq
n^{c}\lg^2 n$ rows sampled random from the complete truth table, within any
polynomial number of iterations.
\end{lemma}
\begin{proof}	
Recall that when sampling $s$ rows uniformly at random from the complete truth
table to evaluate the solution fitness, \RLSGP terminates and returns
the current solution when the solution appears wrong on at most $c' \lg n$ of
the sampled rows. As the generalisation error of a solution is also
the probability that the solution is wrong on a uniformly-sampled row of the
complete truth table, a solution $X$ with a generalisation error $g(X)$
of at least $n^{-c}$, has an expected sampled error $E(f(X)) \geq \lg^2 n$
on $s = n^{c}\lg^2 n$ rows sampled uniformly at random.
Applying a Chernoff bound, the probability that the
sampled error $Y$ is less than half of its expected value (which for
large-enough $n$ is above the $c' \lg n$ threshold), is super-polynomially
small: 
$$ \Pr\left[Y \leq 1/2 \, E[Y]\right] \leq e^{-E[Y]/8} \leq n^{-\Omega(\log n)}. $$

By a union bound, \RLSGP with high probability does not
return a solution with a generalisation error of at least $n^{-c}$ within any
polynomial number of iterations when sampling $s = \Omega(n^{c} \lg^2 n)$ rows
of the complete truth table uniformly at random to evaluate the solution quality in
each iteration.
\end{proof}

Our main result is proved by combining these lemmas.

\begin{proof}[\proofOf{Theorem~\ref{thm:incomplete-single-run}}]
By Lemma~\ref{lem:incomplete-no-bad-things},
Lemma~\ref{lem:incomplete-happy-runtime} can be applied with
probability at least
$1-O(\log^2(n)/n)$, and thus with probability at least $1-O(\log^2(n)/n)-O(1/n)$,
a solution with a sampled error meeting the acceptance threshold will be
found and returned within $O(\log n)$ iterations.
By Lemma~\ref{lem:no-bad-solutions}, the generalisation error of any
solution returned by \RLSGP within a polynomial number of iterations is with
high probability better than the desired $n^{-c}$.
\end{proof}

Performing $\lambda$ runs of \RLSGP, as is often done in practice,
and terminating once any instance determines that its current solution meets the
acceptance threshold, will guarantee that a solution with the desired generalisation
error is produced using $O(\lambda \log n)$ fitness evaluations with probability 
$1 - n^{-\Omega(\lambda)}$.

\newmarker
\peter{We now consider the GP system equipped with complete Boolean function and terminal sets i.e., the algorithm has the capacity
to represent any Boolean function of up to $n$ distinct variables. We show that
w}hen using a polynomial training set to evaluate the solution quality, extending
the terminal set to include negations of input variables (thus allowing \RLSGP
to represent any Boolean function) does not significantly affect the
algorithm's ability to produce solutions with desired polynomially-small
generalisation errors on the \andn and \orn problems. This result is formalised
in Theorem~\ref{thm:incomplete-negated-single-run}, which is proven by
observing that insertions of negated variables are almost as useful as
insertions of positive variables on these problems.

\begin{theorem} \label{thm:incomplete-negated-single-run}
For any constant $c > 0$, consider an instance of the
\RLSGP algorithm with $F = \{AND, OR\}$, $L = \{x_1, \ldots, x_n, \overline{x}_1, \ldots, \overline{x}_n\}$, $\Tmax \geq n$,
using a training set of $s = n^c \lg^2 n$ rows sampled uniformly
at random from the complete truth table in each iteration to evaluate the solution
quality, and terminating when the sampled error of the solution is at most
$c' \lg n$, where $c'$ is an appropriately large constant. For both the
\andn and \orn problems, with probability at least $1 - O(\log^2(n)/n)$, the algorithm will terminate within
$O(\log n)$ iterations, and return a solution with a generalisation
error of at most $n^{-c}$.
\end{theorem}
\begin{proof}
This result can be proven by following the proof of
Theorem~\ref{thm:incomplete-single-run}, supplemented by a number of
observations regarding the utility of adding negated variables into the current
solution.
	
A solution containing both the positive and negative form of some variable and
only problem-appropriate function nodes (i.e., AND for \andn and OR for \orn;
producing a contradiction or a tautology respectively) trivially achieves the
desired polynomially-small generalisation error, as it is only wrong on a
single input out of $2^n$ possible inputs. Typically, however, \RLSGP will
terminate with a solution of desired quality before inserting both the positive
and negative copies of some variable.

Any solution composed of problem-appropriate function nodes and a mix of
positive and negative input variables, but not both the positive and negative
versions of any one variable, is wrong on exactly two more inputs than a
solution which contains only positive versions of the same variables. For the
\andn problem, these two inputs are the all-true input (on which the solution
containing negated variables return false), and the input setting all used leaf
nodes to true by assigning true to variables appearing in positive form, and
false to variables appearing negated, and all remaining variables to true.
Thus, the first insertion of a negated variable (which does not already appear
in the tree in positive form) is only slightly worse than inserting the
corresponding positive variable, while subsequent insertions of negated and
positive variables are equivalent as long as no contradiction or tautology is
produced. The difference is overwhelmingly unlikely to be noticed when sampling
a training set of polynomial size for any polynomial number of iterations:
these two inputs are not selected for any polynomially-sized training set
within a polynomial number of iterations with probability overwhelmingly close
to $1$ by a union bound, and conditioning on this, the \RLSGP is as likely to
accept/reject a mutation involving a negated literal as a positive literal of
the same variable.

For the purposes of Lemmas \ref{lem:incomplete-happy-runtime} and
\ref{lem:incomplete-no-bad-things}, the negated variables can therefore be
treated as being equivalent to their positive counterparts, with any solution
which contains a logarithmic number of distinct variables in either form and
only the problem-appropriate function nodes producing the desired
generalisation error. Thus, with high probability, a tree which contains each
variable in either positive or negative form at most once, contains no
undesired function nodes, and achieves the desired generalisation error can be
constructed within $O(\log n)$ iterations.

Combining these with Lemma~\ref{lem:no-bad-solutions} as in the proof of
Theorem~\ref{thm:incomplete-single-run}, and including the probability of not
sampling a training set capable of distinguishing negations in the final union
bound yields the result.
\end{proof}

\section{Experiments} \label{sec:Experiments}
We performed experiments to complement our theoretical results.
For each choice of algorithm and problem parameters, we performed
500 independent runs of the GP systems and we report the average number of iterations required to find the solution, its average size, as well as standard deviations.

\begin{table}[tp]
	\centering
	\setlength\tabcolsep{1mm}
	\begin{tabular}{|c|c|cc|c|cc|} \hline
    	Experiment & Deletion & $F$ & $L$ & Training Set & Runs & Max Iterations 
	 \\ \hline
 RQ 1 &	leaf  &	$\{AND,OR\}$ &	$\{x_1, \dots, x_n\}$ & Complete & 500 &	10,000	\\
 RQ 2 &	 subtree  &	$\{AND,OR\}$ &	$\{x_1, \dots, x_n\}$ & Complete	& 500 &	10,000 	\\
 RQ 3 &	 subtree &	$\{AND,OR\}$ &	$\{x_1, \ldots, x_n, \overline{x}_1, \ldots, \overline{x}_n\}$ & Complete &	500 &	10,000 	\\
 RQ 4 &	 subtree &$\{AND,OR\}$ &	$\{x_1, \dots, x_n\}$ & Incomplete &	500 &	10,000 	\\
 RQ 5 &	 subtree &$\{AND,OR\}$ &	$\{x_1, \ldots, x_n, \overline{x}_1, \ldots, \overline{x}_n\}$& Incomplete &	500 &	10,000 	\\
 \hline
	\end{tabular} 
\caption{The experimental setup used in each of the five carried out experiments.}
	\label{tab:expsetup}
\end{table}

We will investigate five separate research questions to expand upon the knowledge gained from the theory regarding \RLSGP for evolving conjunctions.
The experimental setup for each research question is reported in Table \ref{tab:expsetup}.

\peter{We first examine the behaviour of the algorithms when they use the complete truth table to evaluate solution quality.}
Theorem \ref{thm:rls-stuck}
showed that using the standard HVL-Prime operator, which applies leaf-only deletion and substitution, can cause \RLSGP with the complete truth table to get stuck on a local
optimum when a tree size limit is imposed, thus leading to infinite expected runtime. However, the theorem does not provide bounds on the probability that this event occurs. 

{\bf Research Question 1}: How likely is it that \RLSGP using leaf-only deletion and the complete truth table gets stuck in a local optimum and how does this probability depend on the tree size limit? 

Table~\ref{tab:leaf-RLS-CTT} summarises the experimental behaviour of \RLSGP.
The experiments confirm that when using small tree size limits, \RLSGP does indeed get stuck on local optima even on small problem sizes, thus motivating the use of sub-tree deletion. Examples of the locally optimal trees constructed during the runs are depicted in Figure \ref{fig:leaf-bad-trees}. However, the probability of getting stuck decreases as $\Tmax$, the limit on the size
of the tree, increases.
Concerning the solution quality, with small tree size limits, the number of redundant variables in the final program decreases at the expense of larger runtimes. For $\Tmax=n$, {\it `exact'} solutions are returned when the algorithm does not get stuck. On the other hand, larger tree size limits (including no limit) lead to smaller expected runtimes at the expense of redundant variables in the produced programs.  

\begin{table}[tp]
	\centering
	\setlength\tabcolsep{1mm}
	\begin{tabular}{|c|ccc|ccc|} \hline
    	& \multicolumn{3}{c}{$\Tmax = n$} & \multicolumn{3}{|c|}{$\Tmax = n+1$} \\
		n & $B$ & $\overline{T}$ & $\overline{S}$ & $B$ & $\overline{T}$ & $\overline{S}$ \\ \hline
 4 &	 0.008 &	46.3 (28.0) &	4.0 (0.0) &	0.002 &	40.9 (21.8) &	4.4 (0.5)	\\
 8 &	 0.002 &	151.8 (91.9) &	8.0 (0.0) &	0.004 &	113.8 (51.5) &	8.6 (0.5)	\\
 12 &	 0.016 &	284.1 (148.2) &	12.0 (0.0) &	0.002 &	214.3 (99.5) &	12.7 (0.5)	\\
 16 &	 0.008 &	469.9 (258.0) &	16.0 (0.0) &	0.010 &	345.8 (161.0) &	16.8 (0.4)	\\
 \hline
		\multicolumn{7}{c}{} \\ \hline
    	&  \multicolumn{3}{c}{$\Tmax = 2n$} & \multicolumn{3}{|c|}{$\Tmax = \infty$} \\
		n & $B$ & $\overline{T}$ & $\overline{S}$ & $B$ & $\overline{T}$ & $\overline{S}$ \\ \hline
 4 &	 0 &	42.5 (25.8) &	5.1 (1.2) &	0 &	38.9 (24.3) &	5.4 (2.0)	\\
 8 &	 0 &	98.8 (49.0) &	11.0 (2.3) &	0 &	95.3 (43.8) &	11.2 (3.0)	\\
 12 &	 0 &	170.7 (99.7) &	17.1 (3.3) &	0 &	160.1 (57.1) &	17.9 (4.5)	\\
 16 &	 0 &	232.5 (80.9) &	23.8 (4.1) &	0 &	235.3 (92.7) &	24.6 (6.0)	\\ \hline
	\end{tabular}
    
\caption{Proportion of runs stuck in a local optimum ($B$), and average runtime
($\overline{T}$) and solution size ($\overline{S}$) of successful runs of the
\RLSGP using leaf-only substitution and deletion with the complete truth table 
to evaluate the solution quality for varying $n$ and $\Tmax$. Standard deviations
appear in parentheses.}
	\label{tab:leaf-RLS-CTT}
\end{table}

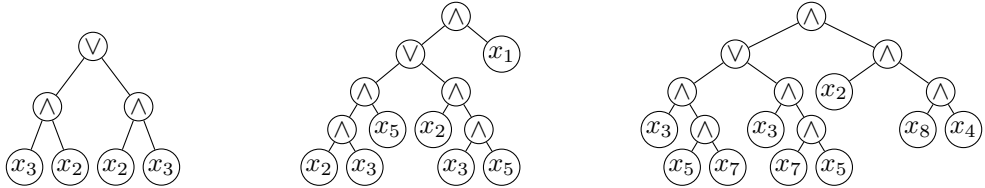
\begin{figure}[tp]
	\hfill
	\begin{tikzpicture}[
		level 1/.style={sibling distance=12mm},
		level 2/.style={sibling distance=6mm}, 
		level distance=8mm,
		every node/.style={draw,circle,inner sep=0.25mm}
	]
	\node {$\vee$}
child { node {$\wedge$}
child { node {$x_3$} }
child { node {$x_2$} } }
child { node {$\wedge$}
child { node {$x_2$} }
child { node {$x_3$} } }
	;
	\end{tikzpicture} \hfill
	\begin{tikzpicture}[
		level 1/.style={sibling distance=12mm},
		level 2/.style={sibling distance=12mm}, 
		level 3/.style={sibling distance=6mm}, 
		level distance=5mm,
		every node/.style={draw,circle,inner sep=0.25mm}
	]
	\node {$\wedge$}
child { node {$\vee$}
child { node {$\wedge$}
child { node {$\wedge$}
child { node {$x_2$} }
child { node {$x_3$} } }
child { node {$x_5$} } }
child { node {$\wedge$}
child { node {$x_2$} }
child { node {$\wedge$}
child { node {$x_3$} }
child { node {$x_5$} } } } }
child { node {$x_1$} };
	\end{tikzpicture} \hfill
    \begin{tikzpicture}[
		level 1/.style={sibling distance=20mm},
		level 2/.style={sibling distance=14mm}, 
		level 3/.style={sibling distance=6mm}, 
		level distance=5mm,
		every node/.style={draw,circle,inner sep=0.25mm}
	]
\node {$\wedge$}
child { node {$\vee$}
child { node {$\wedge$}
child { node {$x_{3}$} }
child { node {$\wedge$}
child { node {$x_{5}$} }
child { node {$x_{7}$} } } }
child { node {$\wedge$}
child { node {$x_{3}$} }
child { node {$\wedge$}
child { node {$x_{7}$} }
child { node {$x_{5}$} } } } }
child { node {$\wedge$}
child { node {$x_{2}$} }
child { node {$\wedge$}
child { node {$x_{8}$} }
child { node {$x_{4}$} } } };
\end{tikzpicture} \hspace*{\fill}
	
	\caption{Examples of locally optimal trees, which cannot be improved by substitution or have any single leaf deleted without affecting fitness, constructed by \RLSGP using leaf-only substitution and deletion operations.}
	\label{fig:leaf-bad-trees}
\end{figure}

We now turn our attention to the 
HVL-Prime operator modified to allow subtree deletion, as considered by
Theorem~\ref{thm:CTT-runtime}. 
The theorem states that the algorithm will find the optimal solution using the complete truth table if an appropriate tree size limit is in place. However, it is unclear from the theory whether the tree size limit is necessary.

{\bf Research Question 2:} How does the tree size limit affect the runtime and solution quality of RLS-GP using subtree deletion and the complete truth table for evolving the conjunction?

As predicted by the theory, \RLSGP never gets stuck in our
experiments when using the complete truth table and a tree size limit. Table~\ref{tab:tree-RLS-CTT} shows the average number of
iterations required to find the global optimum for various problem sizes
and varying tree size limits.
Once again the experiments show that smaller tree size limits lead to lower numbers of redundant variables at the expense of a higher runtime. Larger limits, including no limit at all, lead to faster runtimes at the expense of admitting more redundant variables. Noting that in practical applications a tree size limit is often necessary,
we leave the proof that the algorithm evolves an exact conjunction without any limits on the tree size for future work.

We now add the negated literals to the terminal set to turn our attention to the performance of \RLSGP when equipped with a complete set of functions  and terminals that allow to express any Boolean function. Since it has been proven that negations do not allow to efficiently evolve conjunctions when only the AND function is available to \RLSGP  using the complete truth table~\citep{LissovoiO19}, we have conjectured that the same applies when also the OR function is available to the algorithm.

{\bf Research Question 3:} Can RLS-GP using the fully expressive set of functions and literals evolve the conjunction using the complete truth table?

{\alc
Table~\ref{tab:tree-RLS-CTT-neg} shows the effect of allowing negated literals:
as the problem size increases, \RLSGP is overwhelmingly likely to construct a
negation and not find the global optimum within the allotted runtime of 10,000
iterations} \peter{as we have conjectured.}

\begin{table}[tp]
	\centering
	\setlength\tabcolsep{1.5mm}
	\begin{tabular}{|c|cc|cc|cc|cc|} \hline
    	& \multicolumn{2}{c|}{$\Tmax = n$} & \multicolumn{2}{c}{$\Tmax = n+1$} & \multicolumn{2}{|c}{$\Tmax = 2n$} & \multicolumn{2}{|c|}{$\Tmax = \infty$} \\
		n & $\overline{T}$ & $\overline{S}$ & $\overline{T}$ & $\overline{S}$ & $\overline{T}$ & $\overline{S}$ & $\overline{T}$ & $\overline{S}$ \\ \hline
 4 &	 51.2 & 4.0 &	 42.5 & 4.4 &	 38.8 & 5.1 &	 39.1 & 5.3 \\
 &	 (31.1) & (0.0) &	 (23.5) & (0.5) &	 (20.8) & (1.2) &	 (22.3) & (1.8) \\ \hline
 8 &	 147.5 & 8.0 &	 129.9 & 8.7 &	 93.5 & 11.3 &	 92.3 & 11.6 \\
 &	 (83.3) & (0.0) &	 (69.1) & (0.5) &	 (39.1) & (2.4) &	 (38.1) & (3.0) \\ \hline
 12 &	 325.9 & 12.0 &	 233.4 & 12.8 &	 153.6 & 17.7 &	 151.2 & 18.3 \\
 &	 (184.4) & (0.0) &	 (123.9) & (0.4) &	 (56.6) & (3.1) &	 (50.3) & (3.8) \\ \hline
 16 &	 544.6 & 16.0 &	 377.0 & 16.9 &	 228.3 & 24.5 &	 221.0 & 25.2 \\
 &	 (333.8) & (0.0) &	 (176.0) & (0.4) &	 (74.6) & (3.7) &	 (72.0) & (4.9) \\ \hline
	\end{tabular}
	\caption{
Average runtime ($\overline{T}$) and solution size ($\overline{S}$) of 
\RLSGP using the subtree deletion sub-operation, and the complete truth table 
to evaluate the solution fitness, for varying $n$ and $\Tmax$. Standard deviations
appear in parentheses.
}
	\label{tab:tree-RLS-CTT}
\end{table}

\begin{table}[tp]
	\centering
	\setlength\tabcolsep{1mm}
	\begin{tabular}{|c|ccc|ccc|} \hline
    	& \multicolumn{3}{c}{$\Tmax = n$} & \multicolumn{3}{|c|}{$\Tmax = n+1$} \\
		n & $B$ & $\overline{T}$ & $\overline{S}$ & $B$ & $\overline{T}$ & $\overline{S}$ \\ \hline
4 & 0.000 & 655.3 (638.0) & 4.0 (0.0) & 0.000 & 513.7 (480.8) & 4.5 (0.5) \\ 
8 & 0.998 & 1750.0 (-) & 8.0 (-) & 0.994 & 5535.3 (3066.0) & 8.3 (0.6) \\ 
12 & 1.000 & - & - & 1.000 & - & - \\ 
16 & 1.000 & - & - & 1.000 & - & - \\ 
 \hline
		\multicolumn{7}{c}{} \\ \hline
    	&  \multicolumn{3}{c}{$\Tmax = 2n$} & \multicolumn{3}{|c|}{$\Tmax = \infty$} \\
		n & $B$ & $\overline{T}$ & $\overline{S}$ & $B$ & $\overline{T}$ & $\overline{S}$ \\ \hline
4 & 0.000 & 431.0 (392.7) & 5.7 (1.3) & 0.000 & 422.7 (425.7) & 6.4 (2.5) \\ 
8 & 0.990 & 6443.8 (3789.4) & 11.0 (1.6) & 0.982 & 4478.7 (2371.6) & 12.4 (3.3) \\ 
12 & 1.000 & - & - & 1.000 & - & - \\ 
16 & 1.000 & - & - & 1.000 & - & - \\ \hline
	\end{tabular}
    
\caption{\alc Proportion of runs failing to find the global optimum within $10^4$ iterations ($B$), and average runtime
($\overline{T}$) and solution size ($\overline{S}$) of successful runs of the
\RLSGP using the subtree deletion sub-operation, negated literals, and the complete truth table 
to evaluate the solution fitness, for varying $n$ and $\Tmax$. Standard 
deviations, where available, appear in parentheses.}
	\label{tab:tree-RLS-CTT-neg}
\end{table}

Finally, we examine the behaviour of \RLSGP when using an incomplete
training set \peter{which allows us to evolve bigger conjunctions i.e.,} larger problem sizes.
The results from Theorems~\ref{thm:incomplete-single-run} and~\ref{thm:incomplete-negated-single-run} rely on the algorithm stopping once a logarithmic sampled error is achieved using $F=\{AND,OR\}$ respectively without and with negated literals. However, from the theory it is unclear whether this error threshold is necessary or not.

{\bf Research Question 4:} Is the logarithmic error threshold required by RLS-GP using a training set to evolve the conjunction and how does removing the threshold affect the algorithm's performance using only positive literals?

We run experiments comparing the performance of \RLSGP when stopping at error 0 or stopping earlier for $n=50$. The average runtimes of the two variants are plotted in Figure~\ref{fig:incomplete-training-results}. The figure confirms our theoretical result that the algorithms generally run in logarithmic time and produce solutions that
contain a logarithmic number of leaf nodes with respect to the training set size. Stopping at 0 error leads to better solutions at the expense of higher runtimes.

We now turn our attention to negated literals.

{\bf Research Question 5:} Is the logarithmic error threshold required by RLS-GP using a training set to evolve the conjunction and how does removing the threshold affect the algorithm's performance using both positive and negated literals?

{\alc
Figure~\ref{fig:incomplete-training-results-neg} explores the effect of
expanding the literal set with the negated variables. 
For $n=50$, negations cause a
contradiction to be constructed when a solution contains approximately $8$ leaf
nodes on average, allowing the algorithm to terminate within an average of 50
iterations even for larger training sets.
Figures~\ref{fig:incomplete-training-ORs} and \ref{fig:incomplete-training-ORs-neg}
 show the average number of ORs in the
final program with and without negations in the literal set respectively. While
these are few in number, more disjunctions are present the closer the threshold
on acceptable sampled error is to 0.
}

\begin{figure}[p]
\pgfplotstableread{data/runtime-n50-linf-s0.tsv}\runFiftyInfZero
\pgfplotstableread{data/treesize-n50-linf-s0.tsv}\sizeFiftyInfZero
\pgfplotstableread{data/runtime-n50-linf-s16.tsv}\runFiftyInfSixteen
\pgfplotstableread{data/treesize-n50-linf-s16.tsv}\sizeFiftyInfSixteen
\pgfplotstableread{data/runtime-n50-linf-s8.tsv}\runFiftyInfEight
\pgfplotstableread{data/treesize-n50-linf-s8.tsv}\sizeFiftyInfEight
\pgfplotstableread{data/runtime-n50-linf-s32.tsv}\runFiftyInfThirtytwo
\pgfplotstableread{data/treesize-n50-linf-s32.tsv}\sizeFiftyInfThirtytwo
\begin{tikzpicture}
\begin{semilogxaxis}[
width=0.48\linewidth, height=0.52\linewidth,
ymax=125,
legend pos = north west,
xlabel={Training Set Size (inputs)},
ylabel={Runtime (iterations)},
try min ticks=5,
legend cell align={left},
legend style={nodes={scale=0.8, transform shape}}, 
]
\addplot[cyan,mark=x] table[x index=0, y index = {2}]{\runFiftyInfZero};
\addplot[dgreen,mark=triangle] table[x index=0,y index={1}]{\runFiftyInfEight};
\addplot[orange,mark=o] table[x index=0, y index = {2}]{\runFiftyInfSixteen};
\addplot[dred,mark=square,mark size=1.25pt] table[x index=0,y index={1}]{\runFiftyInfThirtytwo};
\legend{$A=0$,$A=8$,$A=16$,$A=32$}
\end{semilogxaxis}
\end{tikzpicture}\hfill
\begin{tikzpicture}
\begin{semilogxaxis}[
width=0.48\linewidth, height=0.52\linewidth,
ymax=20,
legend pos = north west,
xlabel={Training Set Size (inputs)},
ylabel={Solution Size (leaf nodes)},
try min ticks=5,
legend cell align={left},
legend style={nodes={scale=0.8, transform shape}}, 
]
\addplot[cyan,mark=x] table[x index=0, y index = {2}]{\sizeFiftyInfZero};
\addplot[dgreen,mark=triangle] table[x index=0,y index={1}]{\sizeFiftyInfEight};
\addplot[orange,mark=o] table[x index=0, y index = {2}]{\sizeFiftyInfSixteen};
\addplot[dred,mark=square,mark size=1.25pt] table[x index=0,y index={1}]{\sizeFiftyInfThirtytwo};
\legend{$A=0$,$A=8$,$A=16$,$A=32$}
\end{semilogxaxis}
\end{tikzpicture}
\caption{Average runtime and tree size produced by \RLSGP with subtree deletion, using an incomplete training set, stopping once sampled error is at most $A$, $n=50, \Tmax = \infty$, averaged over 500 independent runs.}
\label{fig:incomplete-training-results}
\end{figure}
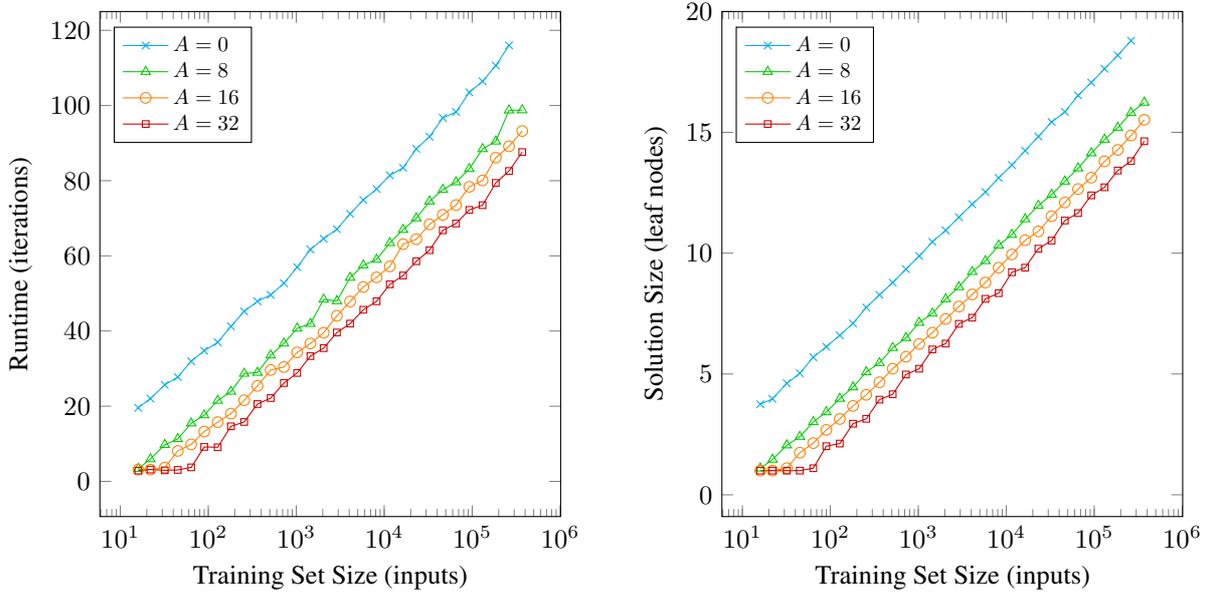
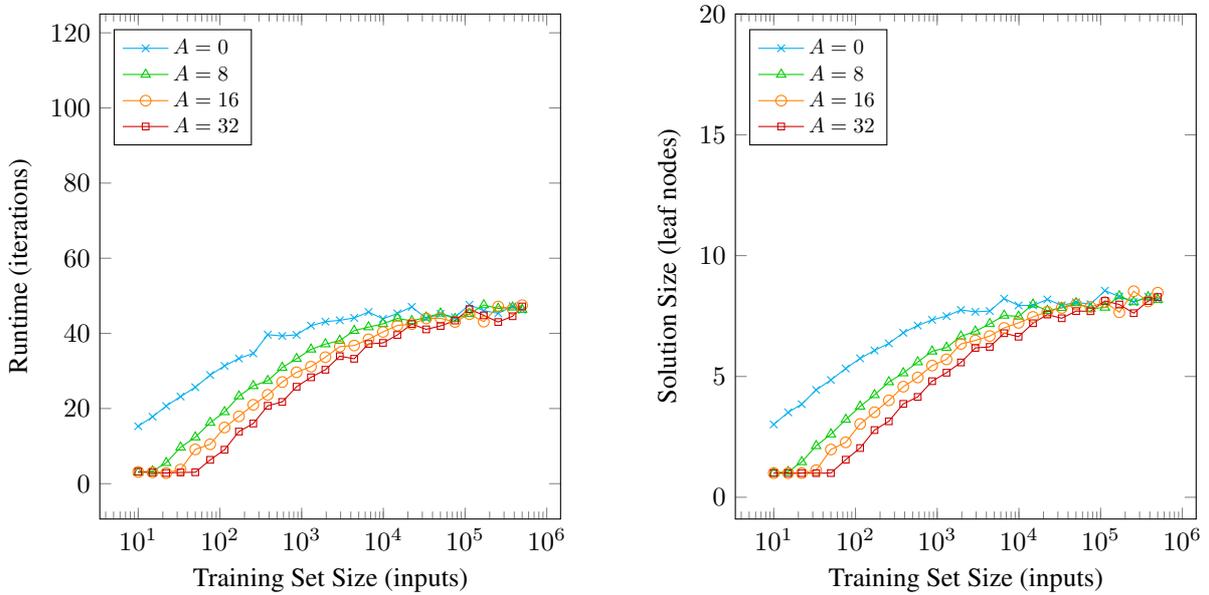
\begin{figure}[tp]
\pgfplotstableread{data/fig2x.tsv}\dat
\begin{tikzpicture}
\begin{semilogxaxis}[
 width=0.48\linewidth, height=0.52\linewidth,
 ymax=125,
 legend pos = north west,
 xlabel={Training Set Size (inputs)},
 ylabel={Runtime (iterations)},
 try min ticks=5,
 legend cell align={left},
 legend style={nodes={scale=0.8, transform shape}}, 
]
\addplot[cyan,mark=x]                         table[x index=0,y index=1] {\dat};
\addplot[dgreen,mark=triangle]                table[x index=0,y index=3] {\dat};
\addplot[orange,mark=o]                       table[x index=0,y index=5] {\dat};
\addplot[dred,mark=square,mark size=1.25pt]   table[x index=0,y index=7] {\dat};
\legend{$A=0$,$A=8$,$A=16$,$A=32$}
\end{semilogxaxis}
\end{tikzpicture}\hfill
\begin{tikzpicture}
\begin{semilogxaxis}[
 width=0.48\linewidth, height=0.52\linewidth,
 ymax=20,
 legend pos = north west,
 xlabel={Training Set Size (inputs)},
 ylabel={Solution Size (leaf nodes)},
 try min ticks=5,
 legend cell align={left},
 legend style={nodes={scale=0.8, transform shape}}, 
]
\addplot[cyan,mark=x]                         table[x index=0,y index=2] {\dat};
\addplot[dgreen,mark=triangle]                table[x index=0,y index=4] {\dat};
\addplot[orange,mark=o]                       table[x index=0,y index=6] {\dat};
\addplot[dred,mark=square,mark size=1.25pt]   table[x index=0,y index=8] {\dat};
\legend{$A=0$,$A=8$,$A=16$,$A=32$}
\end{semilogxaxis}
\end{tikzpicture}
\caption{\alc Average runtime and tree size produced by \RLSGP with subtree deletion and negated literals, using an incomplete training set, stopping once sampled error is at most $A$, $n=50, \Tmax = \infty$, averaged over 500 independent runs.}
\label{fig:incomplete-training-results-neg}
\end{figure}
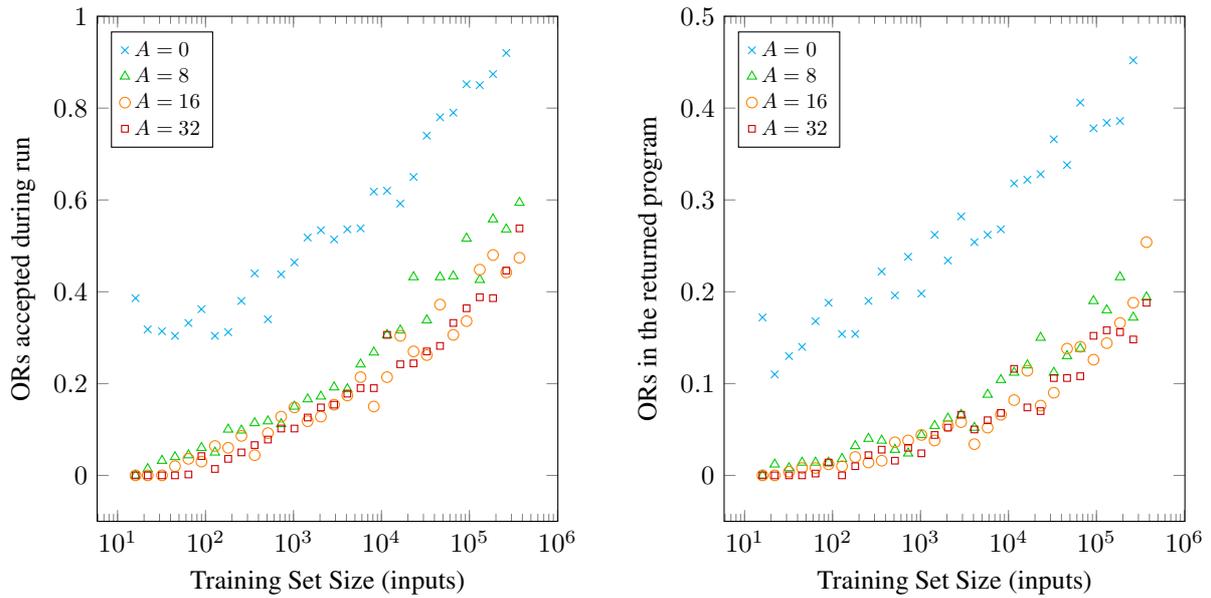
\begin{figure}[tp]
\pgfplotstableread{data/insOR-n50-linf-s0.tsv}\insFiftyInfZero
\pgfplotstableread{data/finOR-n50-linf-s0.tsv}\finFiftyInfZero
\pgfplotstableread{data/insOR-n50-linf-s16.tsv}\insFiftyInfSixteen
\pgfplotstableread{data/finOR-n50-linf-s16.tsv}\finFiftyInfSixteen
\pgfplotstableread{data/insOR-n50-linf-s8.tsv}\insFiftyInfEight
\pgfplotstableread{data/finOR-n50-linf-s8.tsv}\finFiftyInfEight
\pgfplotstableread{data/insOR-n50-linf-s32.tsv}\insFiftyInfThirtytwo
\pgfplotstableread{data/finOR-n50-linf-s32.tsv}\finFiftyInfThirtytwo
\begin{tikzpicture}
\begin{semilogxaxis}[
width=0.48\linewidth, height=0.52\linewidth,
ymax=1,
legend pos = north west,
xlabel={Training Set Size (inputs)},
ylabel={ORs accepted during run},
try min ticks=5,
legend cell align={left},
legend style={nodes={scale=0.8, transform shape}}, 
]
\addplot[cyan,mark=x,only marks] table[x index=0, y index = {2}]{\insFiftyInfZero};
\addplot[dgreen,mark=triangle,only marks] table[x index=0,y index={1}]{\insFiftyInfEight};
\addplot[orange,mark=o,only marks] table[x index=0 y index = {2}]{\insFiftyInfSixteen};
\addplot[dred,mark=square,only marks,mark size=1.25pt] table[x index=0,y index={1}]{\insFiftyInfThirtytwo};
\legend{$A=0$,$A=8$,$A=16$,$A=32$}
\end{semilogxaxis}
\end{tikzpicture}\hfill
\begin{tikzpicture}
\begin{semilogxaxis}[
width=0.48\linewidth, height=0.52\linewidth,
ymax=0.5,
legend pos = north west,
xlabel={Training Set Size (inputs)},
ylabel={ORs in the returned program},
try min ticks=5,
legend cell align={left},
legend style={nodes={scale=0.8, transform shape}}, 
]
\addplot[cyan,mark=x,only marks] table[x index=0, y index = {2}]{\finFiftyInfZero};
\addplot[dgreen,mark=triangle,only marks] table[x index=0,y index={1}]{\finFiftyInfEight};
\addplot[orange,mark=o,only marks] table[x index=0, y index = {2}]{\finFiftyInfSixteen};
\addplot[dred,mark=square,only marks,mark size=1.25pt] table[x index=0,y index={1}]{\finFiftyInfThirtytwo};
\legend{$A=0$,$A=8$,$A=16$,$A=32$}
\end{semilogxaxis}
\end{tikzpicture}
\caption{Number of OR nodes inserted and surviving to the solution
returned by \RLSGP with subtree deletion, using an incomplete training set, stopping once sampled error is at most $A$, $n=50, \Tmax = \infty$, averaged over 500 independent runs. 
}
\label{fig:incomplete-training-ORs}
\end{figure}
\begin{figure}[tp]
\pgfplotstableread{data/fig3x.tsv}\dat
\begin{tikzpicture}
\begin{semilogxaxis}[
width=0.48\linewidth, height=0.52\linewidth,
ymax=1,
legend pos = north west,
xlabel={Training Set Size (inputs)},
ylabel={ORs accepted during run},
try min ticks=5,
legend cell align={left},
legend style={nodes={scale=0.8, transform shape}},
every axis plot/.append style={only marks},
]
\addplot[cyan,mark=x]                         table[x index=0,y index=2] {\dat};
\addplot[dgreen,mark=triangle]                table[x index=0,y index=4] {\dat};
\addplot[orange,mark=o]                       table[x index=0,y index=6] {\dat};
\addplot[dred,mark=square,mark size=1.25pt]   table[x index=0,y index=8] {\dat};
\legend{$A=0$,$A=8$,$A=16$,$A=32$}
\end{semilogxaxis}
\end{tikzpicture}\hfill
\begin{tikzpicture}
\begin{semilogxaxis}[
width=0.48\linewidth, height=0.52\linewidth,
ymax=0.5,
legend pos = north west,
xlabel={Training Set Size (inputs)},
ylabel={ORs in the returned program},
try min ticks=5,
legend cell align={left},
legend style={nodes={scale=0.8, transform shape}}, 
every axis plot/.append style={only marks},
y tick label style={/pgf/number format/fixed}
]
\addplot[cyan,mark=x]                         table[x index=0,y index=1] {\dat};
\addplot[dgreen,mark=triangle]                table[x index=0,y index=3] {\dat};
\addplot[orange,mark=o]                       table[x index=0,y index=5] {\dat};
\addplot[dred,mark=square,mark size=1.25pt]   table[x index=0,y index=7] {\dat};

\legend{$A=0$,$A=8$,$A=16$,$A=32$}
\end{semilogxaxis}
\end{tikzpicture}
\caption{
\alc
Number of OR nodes inserted and surviving to the solution
 returned by \RLSGP with subtree deletion and negated literals, using an incomplete training set, stopping once sampled error is at most $A$, $n=50, \Tmax = \infty$, averaged over 500 independent runs. 
 }
\label{fig:incomplete-training-ORs-neg}
\end{figure}
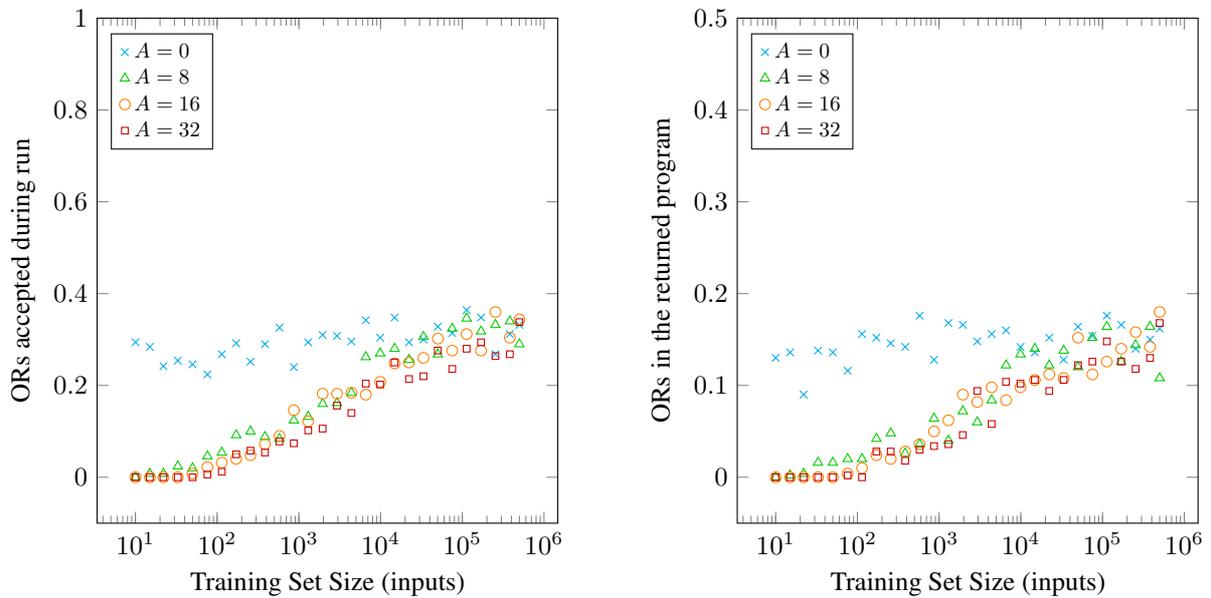

\section{Conclusion} \label{sec:Conclusion}

We analysed the behaviour of a variant of the \RLSGP algorithm and
proved expected runtime bounds when using the complete truth table to evaluate
the solution quality, as well as when using a polynomial sample of possible inputs
chosen uniformly at random. Equipped with a tree size limit and a mutation operator
capable of deleting entire subtrees, \RLSGP is able to efficiently evolve a Boolean 
function -- \andn, the conjunction of $n$ variables \peter{(or \orn a disjunction of $n$ variables)} -- when given access to both
the binary conjunction and disjunction operators. 

When using the complete truth table to evaluate the quality of solutions, we show
that in expectation, an optimal solution is found within $O(\ell n \log^2 n)$ 
iterations \peter{if the terminal set contains exactly the distinct literals contained in the target function}. 
Experimentally, we see that the GP system is able to find solutions
quicker as $\ell$, the limit on the tree size, increases, suggesting that the
theoretical bound is overly pessimistic in its modelling of the process. Conversely,
solutions with larger tree size limits tend to contain more redundant variables,
suggesting a trade-off between optimisation time and solution complexity.
\peter{However, if unnecessary negated literals are also present in the terminal set, then we show experimentally that the algorithm gets stuck on a local optimum (i.e., a contradiction) with overwhelming probability. We leave a formal proof for future work.}

When sampling a polynomial number of inputs to evaluate the program quality,
the evolved solutions are not exactly equivalent to the target function,
but generalise well: any polynomially small
generalisation error can be achieved by sampling a polynomial number of
inputs uniformly at random in each iteration. 
\peter{This result holds even if the GP system is completely expressive i.e., it can represent any Boolean function of at most $n$ distinct variables}.
Our theoretical results predict that
\RLSGP is usually able to avoid inserting ORs in this setting, which is reflected in
our experimental results.
\peter{However, when the unnecessary negated literals are present, then a contradiction is inserted in the tree with high probability leading to solutions of smaller average sizes i.e., the OFF function, that always returns false, may actually be evolved. Since it returns the correct output on all inputs but one, it generalises well nevertheless.} 

While these results represent a considerable step forward for the theoretical
analysis of GP behaviour, much work remains to be done. In addition to
addressing the open problem of removing the limit on the tree size, the
analysis should be extended to cover 
\peter{the evolution of Boolean conjunctions and disjunctions of arbitrary size i.e., not necessarily using all of the $n$ available variables}. 
\peter{Furthemore,} more complex target function\peter{s should be considered} where \peter{the use of} populations and crossover may
be \peter{essential}.

\paragraph*{Acknowledgements}  Financial support by the Engineering and
Physical Sciences Research Council (EPSRC Grant No. EP/M004252/1) is gratefully
acknowledged. This work was also supported by a public grant as part of the
Investissements d'avenir project, reference ANR-11-LABX-0056-LMH,
LabEx LMH.

\bibliographystyle{apalike}
\bibliography{gp}

\end{document}